\documentclass{article}

\usepackage[preprint]{neurips_2026}

\usepackage[utf8]{inputenc}
\usepackage[T1]{fontenc}
\usepackage{hyperref}
\usepackage{url}

\usepackage{booktabs}
\usepackage{microtype}
\usepackage{graphicx}
\usepackage{subcaption}
\usepackage{pifont}
\usepackage{enumitem}
\usepackage{amsmath}
\usepackage{amssymb}
\usepackage{mathtools}
\usepackage{amsthm}
\usepackage{algorithm}
\usepackage[noend]{algpseudocode}
\usepackage[most]{tcolorbox}
\tcbuselibrary{breakable, listings}
\usepackage{colortbl}
\usepackage{float}
\usepackage[capitalize,noabbrev]{cleveref}

\theoremstyle{plain}
\newtheorem{theorem}{Theorem}[section]
\newtheorem{proposition}[theorem]{Proposition}

\theoremstyle{definition}

\theoremstyle{remark}

\definecolor{studentbg}{HTML}{f8cecc}
\definecolor{teacherbg}{HTML}{d5e8d4}

\title{Self-Supervised On-Policy Distillation\\ for Reasoning Language Models}

\author{%
  Zhiquan Tan\\
  Tsinghua University\\
  \And
  Yinrong Hong\\
  Beihang University\\
}
\begin{document}

\maketitle

\begin{abstract}
GRPO-style RLVR trains reasoning models from multiple on-policy attempts per prompt, but typically uses these attempts only through terminal rewards. We show that a mixed group contains a richer process signal: a correct completion is a self-generated witness of how the current policy can solve the problem, while a wrong completion provides on-policy prefixes where the policy needs correction. We introduce \emph{Self-Supervised On-Policy Distillation} (SSOPD), which distills a teacher distribution conditioned on the shortest correct completion into prefixes of the longest wrong completion. This converts intra-group correct--wrong contrast into dense process supervision without external solution traces. A stopping-time view motivates the shortest-correct / longest-wrong rule as a finite-group approximation to editing persistent failures toward fast-success actions, and a prompt-level frontier weight concentrates the auxiliary loss where correct and wrong branches coexist. Across AIME 2024, AIME 2025, and HMMT 2025, SSOPD improves over GRPO in all nine model-benchmark settings. On Qwen3-8B, it reaches a macro Avg@12 of \(65.6\), outperforming GRPO by \(1.6\) points and the solution-conditioned OPSD baseline by \(0.8\) points. Code will be released at \url{https://github.com/tzq1999/SSOPD}.
\end{abstract}

\section{Introduction}

Verifiable reasoning has made it possible to train language models with very little process annotation: sample several solutions, verify their final answers, and reinforce the successful ones. This recipe is particularly effective for mathematical reasoning, where the final answer can often be checked reliably even when the reasoning trace is long and diverse~\citep{wei2022chain,cobbe2021training,hendrycks2021measuring}. Its strength, however, is also its limitation. GRPO-style RLVR methods~\citep{shao2024deepseekmath,guo2025deepseek,yu2025dapo} are fully on-policy, but they compress each sampled reasoning trajectory into a terminal reward. A trajectory may contain many consequential intermediate decisions; after verification, most of that process information is no longer used directly.

Dense supervision offers a complementary advantage. Supervised fine-tuning and distillation can shape token-level behavior and provide local training targets~\citep{guha2025openthoughtsdatarecipesreasoning,xiao2026mimov2flashtechnicalreport,hinton2015distillingknowledgeneuralnetwork,agarwal2024policy,lu2025onpolicydistillation}. Yet this density usually comes from outside the on-policy training loop: reference solutions, privileged traces, or a separate teacher. This creates an information-budget gap. Outcome-based RL is scalable and on-policy but sparse; trace-based supervision is dense but relies on additional information. The question we ask is whether dense process supervision can be extracted from the on-policy samples themselves.

Our key observation is that a GRPO group is already a small empirical reasoning tree for the current policy. The same model attempts the same prompt multiple times, and a mixed group contains a useful contrast: a correct completion demonstrates that the current policy can solve the prompt, while an incorrect completion exposes prefixes that the policy actually visits before failing. Standard GRPO uses this contrast only through group-relative scalar advantages. We instead view it as an opportunity for self-supervision: use the successful attempt as a witness, and use the failed attempt to identify where local guidance should be applied.

\begin{figure*}[t]
\centering
\includegraphics[width=\linewidth]{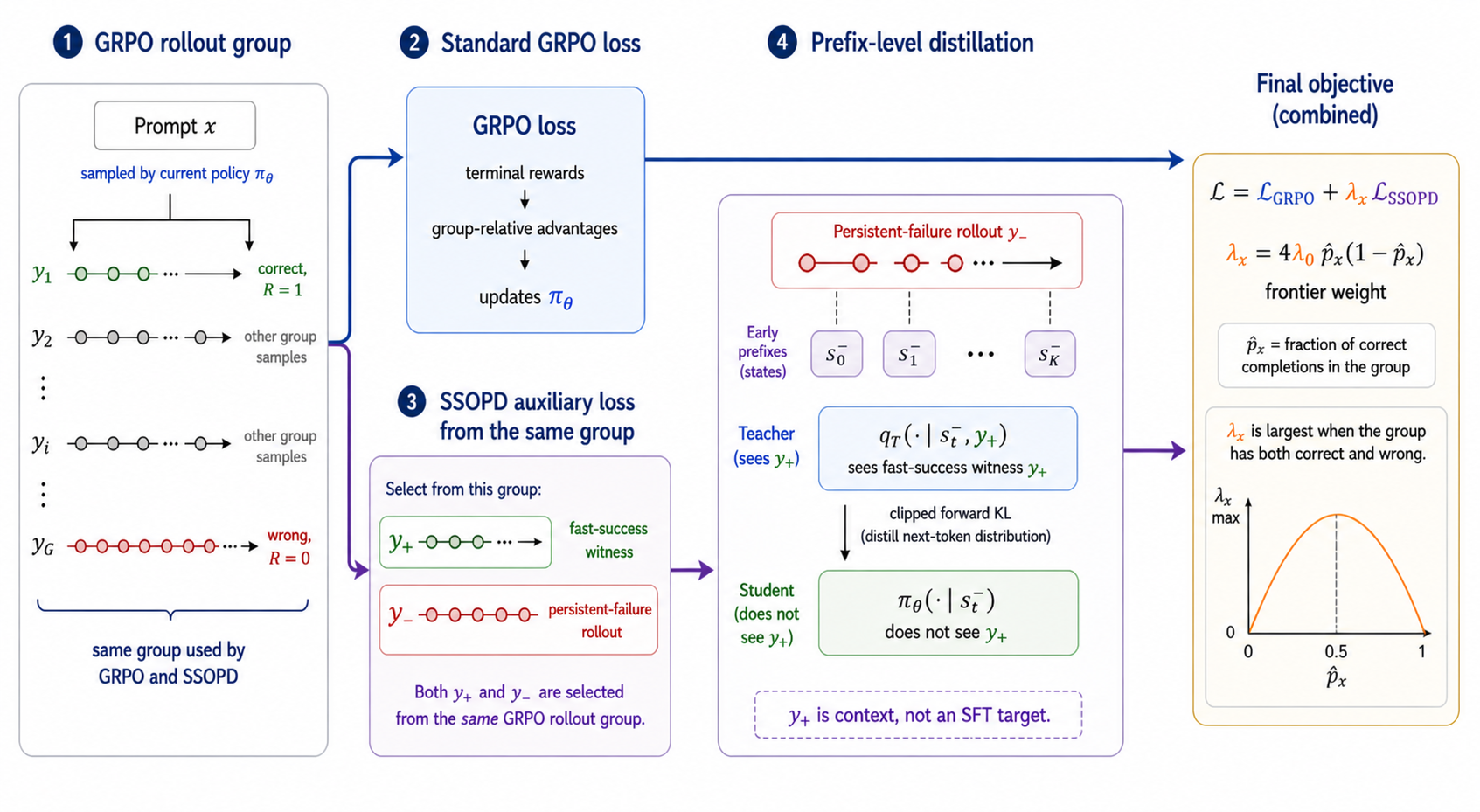}
\caption{\textbf{SSOPD training pipeline.} A GRPO group provides both successful and failed on-policy completions. SSOPD selects a self-generated successful witness, applies a teacher distribution at prefixes of a failed completion, and distills this local distribution into the student with a prompt-level frontier weight.}
\label{fig:mainfig}
\end{figure*}

We instantiate this idea as \emph{Self-Supervised On-Policy Distillation} (SSOPD), shown in \Cref{fig:mainfig}. At a failed prefix, the student is evaluated exactly as it would be at deployment: conditioned only on the prompt and its current prefix. A stop-gradient teacher is evaluated at the same prefix, but is additionally given a successful completion sampled by the same policy for the same prompt. The successful completion is not copied as a target sequence. It serves as privileged hindsight for constructing a local next-token distribution at an on-policy state. In this way, SSOPD converts intra-group correct--wrong contrast into dense distillation without external solution traces or a stronger model.

The method uses a deliberately simple pairing rule. SSOPD selects the shortest correct completion as the successful witness and edits prefixes of the longest wrong completion. We justify this rule through a stopping-time view: under a random computational deadline, shorter correct completions are stronger empirical witnesses of fast success, while longer wrong completions more clearly identify persistent failure routes. A prompt-level frontier coefficient further concentrates the auxiliary loss on groups where correct and wrong branches coexist, which is precisely when the group contains usable self-supervised contrast.

Our experiments separate methods by information budget. SFT and OPSD use additional solution information, whereas GRPO and SSOPD use only answer verification and on-policy samples in the update. Under this non-privileged comparison, SSOPD improves over GRPO in all nine model-benchmark settings across Qwen3-1.7B, 4B, and 8B. On Qwen3-8B, SSOPD reaches a macro Avg@12 of \(65.6\), outperforming GRPO by \(1.6\) points and also exceeding the solution-conditioned OPSD baseline by \(0.8\) points. At smaller scales, OPSD remains stronger overall, suggesting that external solution-conditioned teachers are most helpful when the model's own successful witnesses are less reliable.

Our main contributions are summarized below:
\begin{itemize}[leftmargin=*, itemsep=0pt, parsep=0pt]
  \item We propose SSOPD, a self-supervised distillation objective that turns correct--wrong contrast inside an on-policy GRPO group into dense prefix-level supervision.
  \item We provide a local node-editing analysis and a stopping-time interpretation that motivate the shortest-correct / longest-wrong pairing rule and the frontier weighting coefficient.
  \item We evaluate SSOPD on three competition-level math benchmarks and three Qwen3 scales, showing consistent gains over GRPO without additional solution traces and competitive performance against solution-conditioned distillation at 8B.
\end{itemize}

\section{Background}
\label{sec:opsd_background}

We now fix notation and review the two ingredients used by SSOPD. GRPO supplies on-policy groups of attempts together with terminal verifier rewards. On-policy and student-distribution-aware distillation provide a way to construct dense next-token targets at model-visited prefixes~\citep{agarwal2024policy, lu2025onpolicydistillation, xuspeculative, zhao2026self}. SSOPD combines these ideas while changing the source of privileged context: the context comes from another completion in the same group rather than from an external solution trace.

We use one notation throughout the background and method. For prompt $x$, a completion is $y=(a_0,\ldots,a_{\tau(y)-1})$, its prefix state is $s_t(y)=(x,y_{<t})$, and the verifier returns $R(y)\in\{0,1\}$. The behavior policy that samples rollouts is $\pi$, the optimized student is $\pi_\theta$, and the stop-gradient or frozen teacher is $\pi_{\bar\theta}$.

\subsection{Group Relative Policy Optimization}
\label{sec:background_grpo}

Group Relative Policy Optimization (GRPO) is an RLVR algorithm introduced with DeepSeekMath~\citep{shao2024deepseekmath}. For each prompt $x$, it samples a group
\[
\mathcal G(x)=\{y_i\}_{i=1}^G,
\qquad
 y_i\sim\pi(\cdot\mid x),
\]
and computes binary rewards $r_i=R(y_i)$. Let
\[
\bar r=\frac1G\sum_{j=1}^G r_j,
\qquad
\sigma_r^2=\frac1G\sum_{j=1}^G (r_j-\bar r)^2.
\]
GRPO uses the group-normalized advantage
\begin{equation}
A_i=\frac{r_i-\bar r}{\sigma_r+\epsilon_r},
\label{eq:grpo_advantage}
\end{equation}
where $\epsilon_r$ is a small stabilizer. For token $a_{i,t}$ at prefix $s_t(y_i)$, define the policy ratio
\[
\omega_{i,t}(\theta)
=
\frac{\pi_\theta(a_{i,t}\mid s_t(y_i))}{\pi(a_{i,t}\mid s_t(y_i))}.
\]
The denominator is the behavior, or old, policy that generated the sampled group.
Up to the usual sign convention for minimization versus maximization, the token-averaged GRPO loss is
\begin{equation}
\mathcal L_{\mathrm{GRPO}}(x)
=
\frac1G\sum_{i=1}^G \frac1{\tau(y_i)}
\sum_{t=0}^{\tau(y_i)-1}
\left(
-\min\!\left[
\omega_{i,t}(\theta)A_i,
\operatorname{clip}(\omega_{i,t}(\theta),1-\epsilon,1+\epsilon)A_i
\right]
+
\beta\,\mathcal K_{i,t}
\right),
\label{eq:grpo}
\end{equation}
where $\mathcal K_{i,t}$ denotes the reference-policy KL penalty at the same prefix. GRPO is a compact group-based RLVR objective: the group mean and variance turn terminal verifier outcomes into relative advantages, and the clipped ratio controls the size of the policy update. This makes it a natural on-policy backbone for adding auxiliary process-level supervision.

\subsection{Privileged On-Policy Self-Distillation}
\label{sec:background_opsd}

Privileged on-policy self-distillation~\citep{zhao2026self} keeps the student update at an on-policy prefix while constructing a dense teacher distribution from additional training-time context. At a prefix state $s$, let $z$ denote the privileged information. The teacher and student next-action distributions are
\[
q_T(a\mid s,z)
=
\pi_{\bar\theta}\!\left(a\mid \operatorname{Hint}(s,z)\right),
\qquad
p_\theta(a\mid s)=\pi_\theta(a\mid s).
\]
In our implementation, the local distillation loss is the pointwise-clipped forward KL
\begin{equation}
\ell_{\mathrm{OPSD}}(s,z)
=
\sum_a
\operatorname{clip}_{\tau_{\mathrm{clip}}}
\left(
q_T(a\mid s,z)
\left[
\log q_T(a\mid s,z)-\log p_\theta(a\mid s)
\right]
\right),
\label{eq:local_opsd}
\end{equation}
where $\operatorname{clip}_{\tau_{\mathrm{clip}}}(u)=\min(u,\tau_{\mathrm{clip}})$ is applied to each action-level summand before summing over the vocabulary, and we use $\tau_{\mathrm{clip}}=0.05$ unless otherwise stated. Here $s$ remains a prefix produced by the current behavior policy, while $z$ is used only to form the teacher distribution during training.

\section{Self-Supervised On-Policy Distillation}
\label{sec:ssopd}
\label{sec:method}

A GRPO group is more than a collection of terminal rewards. When the same policy attempts the same prompt several times, the group often contains both evidence of how the model succeeds and states from which it fails. SSOPD turns this contrast into a local training signal. A correct completion is used as privileged context for a stop-gradient teacher, while prefixes of a wrong completion remain the states at which the student is trained. The student never sees the correct completion at inference time; it only learns the next-token distribution that the teacher would prefer at the failing prefix.

The construction follows three steps. First, we describe what a good local edit should do at a prefix. Second, we explain how an on-policy successful completion can approximate such an edit through privileged distillation. Third, we choose the successful and failed completions by a stopping-time criterion and weight the resulting loss by how much correct--wrong contrast the group contains.

\subsection{From GRPO Groups to Local Edits}

Fix a prompt $x$. Let $\pi$ be the behavior policy that samples the GRPO group and let $\pi_\theta$ be the student policy being optimized. A generated completion is a path
\[
Y=(A_0,A_1,\ldots,A_{\tau(Y)-1}),
\qquad
\tau(Y)\le H,
\]
where $\tau(Y)$ is the stopping time and the verifier returns a binary reward $R(Y)\in\{0,1\}$. We write the prefix state at step $t$ as
\[
s_t(y)=(x,y_{<t}).
\]

For each prompt, GRPO samples
\[
\mathcal G(x)=\{y_i\}_{i=1}^G,
\qquad
 y_i\sim\pi(\cdot\mid x),
\]
and splits the group into correct and wrong completions,
\[
\mathcal C(x)=\{y_i\in\mathcal G(x):R(y_i)=1\},
\qquad
\mathcal W(x)=\{y_i\in\mathcal G(x):R(y_i)=0\}.
\]
GRPO uses the rewards in this group to form terminal advantages. SSOPD asks a different question: can the successful members of the group provide dense guidance at prefixes visited by the failed members?

To see what such guidance should look like, consider one prefix $s$ and temporarily ignore finite-sample constraints. Let $Y\sim\pi(\cdot\mid s)$ be a continuation from $s$, and let $\phi_s(Y)\ge0$ score how desirable the completed continuation is. Define
\[
V_\phi^\pi(s)=\mathbb E_\pi[\phi_s(Y)\mid s],
\qquad
Q_\phi^\pi(s,a)=\mathbb E_\pi[\phi_s(Y)\mid s,a].
\]
The natural next-action target is the behavior policy reweighted by the desirability of each action's future:
\begin{equation}
q_\phi(a\mid s)
=
\frac{\pi(a\mid s)Q_\phi^\pi(s,a)}{V_\phi^\pi(s)}.
\label{eq:desirability_posterior}
\end{equation}
This posterior keeps the support of the behavior policy, but shifts mass toward actions whose continuations score better. The following local result formalizes why such a posterior is the right object to imitate.

\begin{theorem}[Local improvement by reweighting good futures]
\label{thm:desirability_edit}
Fix a prefix state $s$ and a nonnegative score $\phi_s$ with $V_\phi^\pi(s)>0$. Define a local edit
\[
\pi_\eta(a\mid s)=(1-\eta)\pi(a\mid s)+\eta q_\phi(a\mid s),
\qquad 0<\eta<1,
\]
leaving all other states unchanged. Then
\[
V_\phi^{\pi_\eta}(s)-V_\phi^\pi(s)
=
\eta
\frac{
\operatorname{Var}_{a\sim\pi(\cdot\mid s)}[Q_\phi^\pi(s,a)]
}{V_\phi^\pi(s)}
\ge0.
\]
\end{theorem}

The result captures the desired behavior of process supervision. If different next actions from $s$ lead to different-quality futures, moving toward $q_\phi$ improves the local value. If all actions are equally good, the variance term is small and there is little to learn at that prefix. Thus the edit naturally focuses on prefixes where the next decision matters. The proof is given in Appendix~\ref{sec:detailed_proofs}.

In practice, we cannot compute \eqref{eq:desirability_posterior}, because it requires all possible continuations from $s$. SSOPD therefore uses a successful sample in the group as a finite-sample proxy. Suppose the group contains a useful witness $z$. We condition a stop-gradient teacher on the same prefix $s$ together with $z$ as extra context:
\begin{equation}
q_T(\cdot\mid s,z)
=
\pi_{\bar\theta}(\cdot\mid \operatorname{Hint}(s,z)),
\qquad
p_\theta(\cdot\mid s)=\pi_\theta(\cdot\mid s).
\label{eq:ssopd_teacher}
\end{equation}
Here $\operatorname{Hint}(s,z)$ packages the prompt-prefix state $s$ together with the self-generated witness $z$. The teacher is allowed to use this hindsight, but the student is not. The auxiliary loss distills the teacher's local decision distribution back into the ordinary student policy. Following the OPSD implementation in \S\ref{sec:background_opsd}, we use the pointwise-clipped forward KL
\begin{equation}
\ell_{\mathrm{OPSD}}(s,z)
=
\sum_a
\operatorname{clip}_{\tau_{\mathrm{clip}}}
\left(
q_T(a\mid s,z)
\left[
\log q_T(a\mid s,z)-\log p_\theta(a\mid s)
\right]
\right).
\label{eq:method_local_opsd}
\end{equation}
The clipping is applied to each action-level summand before summing over the vocabulary, matching \eqref{eq:local_opsd}. Importantly, $z$ is not treated as an SFT target. It only shapes a next-token distribution at the failed prefix.

\subsection{Choosing Witnesses by Stopping Time}

The remaining design question is which completions should instantiate the witness-prefix pair. A useful successful witness should point toward desirable futures, and a useful failed rollout should provide prefixes where correction matters. SSOPD makes both choices with the same stopping-time lens: fast success is a strong positive witness, while long failure is a strong negative witness.

Introduce an independent random deadline $D$ with survival function
\[
\Pr(D\ge \ell)=\gamma^\ell,
\qquad 0<\gamma<1.
\]
The parameter $\gamma$ controls how much we prefer shorter successful continuations: smaller $\gamma$ makes the deadline tighter, while $\gamma$ near one makes the preference milder.

For a continuation from prefix $s$, let $L_s(Y)$ be the remaining stopping time. The event of solving before the deadline gives the fast-success desirability
\begin{equation}
\phi_F^s(Y)=\mathbf 1\{R(Y)=1\}\gamma^{L_s(Y)}.
\label{eq:fast_success_desirability}
\end{equation}
Let $V_F^\pi$ and $Q_F^\pi$ denote the corresponding value and action-value. The fast-success posterior is
\[
q_F(a\mid s)=\frac{\pi(a\mid s)Q_F^\pi(s,a)}{V_F^\pi(s)}.
\]
It favors actions that lead not only to correct answers, but to correct answers reached before the soft deadline. Appendix~\ref{sec:detailed_proofs} first gives a generic approximation result for teacher-induced edits (Theorem~\ref{thm:approx_ssopd_edit}), which we instantiate below for SSOPD.

At the root, a correct completion $y$ has fast-success weight
\[
w_F(y)=\mathbf 1\{R(y)=1\}\gamma^{\tau(y)}.
\]
Within the correct set, this weight is largest for the shortest correct completion:
\begin{equation}
y^+=\arg\min_{y\in\mathcal C(x)}\tau(y).
\label{eq:shortest_correct}
\end{equation}
This choice should not be read as a claim that shorter reasoning is always better. It is the highest-weight observed witness under the fast-success criterion, and it is used only as privileged context for the teacher.

The same deadline identifies persistent failures. A wrong completion receives weight
\begin{equation}
w_P(y)=\mathbf 1\{R(y)=0\}\left(1-\gamma^{\tau(y)}\right).
\label{eq:persistent_failure_weight}
\end{equation}
Among wrong completions, this weight increases with stopping time. A long wrong rollout is therefore strong evidence that its early prefixes lead to an extended failed continuation, which makes those prefixes useful places to intervene. SSOPD selects
\begin{equation}
y^-=\arg\max_{y\in\mathcal W(x)}\tau(y).
\label{eq:longest_wrong}
\end{equation}
Thus the shortest-correct / longest-wrong rule is not an independent heuristic. It is the finite-group hard-witness approximation to editing persistent failed prefixes toward fast-success actions.

\subsection{Empirical Objective and Frontier Weighting}

Once the pair $(y^+,y^-)$ is chosen, the empirical objective follows directly. We edit only the first $K$ generated positions of the failed completion:
\[
K_-=\min(K,\tau(y^-)),
\qquad
s_t^-=(x,y^-_{<t}),
\qquad
t=0,\ldots,K_- -1.
\]
The SSOPD auxiliary loss for prompt $x$ is
\begin{equation}
\mathcal L_{\mathrm{SSOPD}}(x)
=
\frac1{K_-}\sum_{t=0}^{K_- -1}
\ell_{\mathrm{OPSD}}(s_t^-,y^+).
\label{eq:empirical_fp_opsd}
\end{equation}
The roles of the two completions are deliberately asymmetric. The long failed suffix of $y^-$ is not used as token supervision; it certifies that the selected early prefixes are on a poor on-policy route. The correct completion $y^+$ is not copied by the student; it gives the teacher privileged evidence for a better local continuation.

The empirical loss above is a hard-witness approximation, so its usefulness depends on whether the teacher induced by $y^+$ points in a good local direction at the prefixes of $y^-$. For a selected prefix $s_t^-$, define the teacher approximation error to the fast-success posterior as
\[
\varepsilon_{F,s_t^-}(y^+)
=
\left|
\mathbb E_{a\sim q_T(\cdot\mid s_t^-,y^+)}[Q_F^\pi(s_t^-,a)]
-
\mathbb E_{a\sim q_F(\cdot\mid s_t^-)}[Q_F^\pi(s_t^-,a)]
\right|,
\]
where $q_F$ is the posterior induced by the fast-success desirability in \eqref{eq:fast_success_desirability}. The next result states the resulting edit condition.

\begin{theorem}[Stopping-time-aware SSOPD lower bound]
\label{thm:stopping_opsd_bound}
Let $y^+$ and $y^-$ be selected by \eqref{eq:shortest_correct} and \eqref{eq:longest_wrong}. For each $s_t^-$, define the teacher-induced edit
\[
\tilde\pi_\eta^{s_t^-}(a\mid s_t^-)
=
(1-\eta)\pi(a\mid s_t^-)
+
\eta q_T(a\mid s_t^-,y^+).
\]
Then
\[
\frac1{K_-}\sum_{t=0}^{K_- -1}
\left[V_F^{\tilde\pi_\eta^{s_t^-}}(s_t^-)-V_F^\pi(s_t^-)\right]
\ge
\eta\frac1{K_-}\sum_{t=0}^{K_- -1}
\left(
\frac{\operatorname{Var}_{a\sim\pi(\cdot\mid s_t^-)}[Q_F^\pi(s_t^-,a)]}{V_F^\pi(s_t^-)}
-
\varepsilon_{F,s_t^-}(y^+)
\right).
\]
\end{theorem}

The bound mirrors the construction and identifies when the auxiliary edit is well aligned. The selected failed prefix should contain a real branching decision for fast success, and the shortest correct completion should give the teacher useful guidance for that decision. The long wrong completion chooses where to edit; the short correct completion determines what direction the teacher suggests. The proof is in Appendix~\ref{sec:detailed_proofs}.

The auxiliary edit should also be applied only when the group contains meaningful contrast. If the model almost never solves $x$, the group has no reliable successful witness. If the model almost always solves $x$, the few failures are less representative. The most useful regime is the current reasoning frontier, where correct and wrong branches coexist for the same prompt.

Let
\[
p_x=\Pr_\pi(R(Y)=1\mid x),
\qquad
\hat p_x=\frac1G\sum_{i=1}^G R(y_i).
\]
We therefore weight the SSOPD term by
\begin{equation}
\lambda_x=\lambda_0\cdot 4\hat p_x(1-\hat p_x),
\label{eq:frontier_weight}
\end{equation}
treating $\lambda_x$ as stop-gradient. The coefficient is largest for balanced groups and vanishes when the group is all-correct or all-wrong. It also has a useful tree-level interpretation. Let $V_R^\pi(s)=\Pr_\pi(R=1\mid s)$ and $Q_R^\pi(s,a)=\Pr_\pi(R=1\mid s,a)$.

\begin{proposition}[Frontier coefficient as branching variance]
\label{prop:frontier_coefficient}
For binary rewards and a finite horizon with terminal padding,
\[
p_x(1-p_x)
=
\sum_{t=0}^{H-1}
\mathbb E_{S_t}
\left[
\operatorname{Var}_{a\sim\pi(\cdot\mid S_t)}
[Q_R^\pi(S_t,a)]
\right].
\]
Moreover, if $K_c=\sum_i R(y_i)$ and $K_c\sim\operatorname{Binomial}(G,p_x)$, then
\[
\mathbb E[K_c(G-K_c)]=G(G-1)p_x(1-p_x).
\]
\end{proposition}

Thus $p_x(1-p_x)$ is not only a convenient confidence heuristic. It measures how much outcome-relevant branching exists inside the current reasoning tree, and it also predicts how many usable correct--wrong pairs a finite group will contain. This is why the auxiliary loss is strongest near the current reasoning frontier. The proof is in Appendix~\ref{sec:detailed_proofs}.

The final per-prompt objective is
\begin{equation}
\mathcal L(x)
=
\mathcal L_{\mathrm{GRPO}}(x)
+
\lambda_x
\mathcal L_{\mathrm{SSOPD}}(x).
\label{eq:full_method_objective}
\end{equation}
Operationally, when no correct--wrong pair exists, no witness pair is instantiated and $\lambda_x=0$. SSOPD therefore adds dense process supervision in the regime where the on-policy group can support it: the model has found at least one way to succeed, has also exposed a failed route, and can use the former to repair the latter at the prefix level.

\section{Experiments}
\label{sec:experiments}

Our experiments test whether Self-Supervised On-Policy Distillation (SSOPD) can turn the model's own on-policy attempts into useful process supervision. The central comparison is information-controlled. SFT and OPSD use additional solution information, while GRPO and SSOPD use only answer verification and samples from the current policy. We ask whether SSOPD can improve the non-privileged GRPO setting, and how close self-generated supervision can come to solution-conditioned distillation.

\subsection{Experimental Setup}

\textbf{Models and data.} We use the instruct-tuned Qwen3 family~\citep{qwen3technicalreport} at three scales: Qwen3-1.7B, Qwen3-4B, and Qwen3-8B. Training prompts come from the mathematical reasoning subset of OpenThoughts~\citep{guha2025openthoughtsdatarecipesreasoning}; for each scale, we sample up to 30K problems with verified final answers. Reference reasoning traces are used only by the baselines that explicitly require additional solution information, namely SFT and OPSD. Evaluation covers AIME 2024, AIME 2025, and HMMT 2025.

\textbf{Compared post-training recipes.} We evaluate four updates from the same base checkpoints, organized by the information available during training. \textbf{GRPO}~\citep{shao2024deepseekmath} optimizes binary answer rewards with grouped rollouts and uses no reference solution. \textbf{SSOPD} keeps the same non-privileged setting, but adds our auxiliary distillation term: the model's own shortest correct completion serves as the successful witness, prefixes of the longest wrong completion are edited, and the auxiliary loss is weighted by the raw group success rate through $\hat p_x(1-\hat p_x)$. For privileged comparisons, \textbf{SFT} directly imitates reference reasoning traces, and \textbf{OPSD}~\citep{zhao2026self} scores student prefixes with a solution-conditioned teacher.

\textbf{Training and evaluation protocol.} Unless otherwise stated, SSOPD uses a fixed privileged teacher anchored at the initial checkpoint while the student policy is updated with LoRA~\citep{hu2022lora}. Teacher forward passes disable the trainable adapters, and the base coefficient of the SSOPD weight is set to $\lambda_0=0.5$. We report Avg@12 accuracy: for each problem, the model samples 12 completions and is credited according to the average verified correctness. Evaluation uses thinking mode, temperature $1.0$, top-p $0.95$, and a maximum generation length of $38$k tokens. Additional training details appear in Appendix~\ref{sec:experimental_details}.

The protocol is designed to isolate the value of reusing the sampled group itself. GRPO and SSOPD are information-matched: both update from prompts, on-policy completions, and verified final answers. Their difference is how much structure they extract from the same group. SFT and OPSD are included as privileged references because they use additional solution-side information, and therefore measure how close self-supervised group reuse can come to solution-conditioned supervision.

\subsection{Main Results}
\label{sec:main_results}

\begin{table}[t]
\centering
\caption{Math reasoning Avg@12 after post-training Qwen3 backbones. Results are reported in percentage points. SFT and OPSD use additional solution information, while GRPO and SSOPD use only answer verification and on-policy samples. Evaluation uses thinking mode, temperature $1.0$, top-p $0.95$, and maximum length $38$k.}
\vspace{1mm}
\label{tab:main_results}
\resizebox{0.88\textwidth}{!}{
\begin{tabular}{@{}lcccc@{}}
\toprule
\textbf{Backbone / post-training recipe} & \textbf{AIME 2024} & \textbf{AIME 2025} & \textbf{HMMT 2025} & \textbf{Macro avg.} \\
\midrule
\multicolumn{5}{@{}l}{\textit{Qwen3-8B Instruct}} \\
\quad No update              & 75.8 & 65.6 & 43.9 & 61.8 \\
\multicolumn{5}{@{}l}{\quad\textit{No additional solution information}} \\
\quad + GRPO                 & 76.4 & 68.9 & 46.7 & 64.0 \\
\rowcolor{gray!20}\quad + SSOPD (ours)
& \textbf{78.6} & \textbf{70.8} & \textbf{47.5} & \textbf{65.6} \\
\multicolumn{5}{@{}l}{\quad\textit{Uses additional solution information}} \\
\quad + SFT                  & 72.3 & 64.2 & 42.9 & 59.8 \\
\quad + OPSD                 & 77.8 & \textbf{70.8} & 45.8 & 64.8 \\
\midrule
\multicolumn{5}{@{}l}{\textit{Qwen3-4B Instruct}} \\
\quad No update              & 74.9 & 66.4 & 42.2 & 61.2 \\
\multicolumn{5}{@{}l}{\quad\textit{No additional solution information}} \\
\quad + GRPO                 & 75.6 & 68.1 & 44.4 & 62.7 \\
\rowcolor{gray!20}\quad + SSOPD (ours)
& 75.8 & \textbf{68.3} & 45.0 & 63.0 \\
\multicolumn{5}{@{}l}{\quad\textit{Uses additional solution information}} \\
\quad + SFT                  & 70.2 & 62.3 & 43.4 & 58.6 \\
\quad + OPSD                 & \textbf{76.4} & \textbf{68.3} & \textbf{46.1} & \textbf{63.6} \\
\midrule
\multicolumn{5}{@{}l}{\textit{Qwen3-1.7B Instruct}} \\
\quad No update              & 51.5 & 36.7 & 23.1 & 37.1 \\
\multicolumn{5}{@{}l}{\quad\textit{No additional solution information}} \\
\quad + GRPO                 & 51.1 & 38.3 & 23.7 & 37.7 \\
\rowcolor{gray!20}\quad + SSOPD (ours)
& 51.7 & 38.6 & 24.7 & 38.3 \\
\multicolumn{5}{@{}l}{\quad\textit{Uses additional solution information}} \\
\quad + SFT                  & 48.4 & 36.3 & 22.7 & 35.8 \\
\quad + OPSD                 & \textbf{57.2} & \textbf{43.9} & \textbf{29.2} & \textbf{43.4} \\
\bottomrule
\end{tabular}
}
\end{table}

Table~\ref{tab:main_results} shows that SSOPD is the strongest method that does not use additional solution information. It improves over GRPO in every benchmark at every scale: all nine entries increase, with macro-average gains of $+0.6$ on Qwen3-1.7B, $+0.3$ on Qwen3-4B, and $+1.6$ on Qwen3-8B. This is the main empirical takeaway: SSOPD adds a dense, process-level learning signal while keeping the same information budget as reward-only on-policy training.

The 8B result is particularly notable. SSOPD not only improves over GRPO, but also surpasses the solution-conditioned OPSD baseline in macro average ($65.6$ vs. $64.8$), achieves the best AIME 2024 and HMMT 2025 scores, and ties the best AIME 2025 score. In this regime, the model is strong enough to generate reliable successful witnesses, and self-supervision from the group can compete with methods that use external solution information.

The smaller-scale results clarify the boundary of this behavior. On Qwen3-1.7B and Qwen3-4B, OPSD remains stronger overall, indicating that external solution-conditioned teachers are still valuable when the model's own correct completions are less stable. At the same time, SSOPD remains consistently better than GRPO, suggesting that even imperfect self-generated witnesses can improve the on-policy learning signal. SFT degrades on average at all three scales, likely because directly imitating fixed reference traces can shift the model away from its own test-time reasoning distribution.

The benchmark-level pattern is also consistent with the intended use of SSOPD. AIME 2024 and AIME 2025 already have relatively high base accuracy for the larger models, so improvements there test whether the auxiliary loss can refine an already competent policy without adding off-policy traces. HMMT 2025 is harder and leaves more room for persistent failed rollouts, making it a useful stress test for the longest-wrong prefix selection. SSOPD improves over GRPO on all three benchmarks, suggesting that the self-distillation signal is not tied to a single difficulty regime.

\subsection{Ablation on the Frontier Weight}
\label{sec:frontier_weight_ablation}

We ablate the auxiliary coefficient on HMMT 2025 using the Qwen3-4B setting. With the dynamic frontier weight $\lambda_x=\lambda_0\cdot4\hat p_x(1-\hat p_x)$, $\lambda_0\in\{0.4,0.5,0.6,0.7\}$ gives Avg@12 scores of $44.4$, $45.0$, $44.7$, and $44.2$, respectively. Using the same values as fixed global coefficients instead gives $44.2$, $44.7$, $44.4$, and $43.9$. Thus the dynamic coefficient is stable around the default $\lambda_0=0.5$ and improves over fixed weighting at every tested value, supporting its role in allocating distillation to prompts whose sampled groups contain usable correct--wrong contrast.

\section{Related Work}
\label{sec:related_work}

Reasoning post-training with verifiable rewards has become a standard way to improve mathematical reasoning. GRPO-style RLVR methods~\citep{shao2024deepseekmath,guo2025deepseek,yu2025dapo} and recent reasoning post-training systems~\citep{team2025kimi,rastogi2025magistral} optimize answer correctness from sampled trajectories, keeping training close to the model's own rollout distribution. For our purposes, the remaining bottleneck is the granularity of feedback: the verifier usually returns a terminal reward, so many intermediate reasoning decisions receive credit only through a sparse sequence-level signal. SSOPD keeps the same on-policy and verifier-based setting, but extracts an additional dense signal from the internal structure of the group. Rather than changing the reward model or adding step annotations, it uses a successful on-policy sample to supervise failed on-policy prefixes from the same prompt.

Dense supervision for reasoning is often obtained through reference traces, process labels, or distillation. Supervised fine-tuning on curated mathematical traces~\citep{guha2025openthoughtsdatarecipesreasoning,xiao2026mimov2flashtechnicalreport} and process supervision~\citep{lightman2023let,zhang2025lessons} provide more localized targets, but require reference states, annotations, or learned reward models. Distillation and self-improvement methods transfer teacher distributions, generated sequences, or behavior under richer context~\citep{hinton2015distillingknowledgeneuralnetwork,kim2016sequence,sanh2019distilbert,gu2024minillm,zelikman2022star,gulcehre2023reinforced,snell2022learning,huang2022context}. Closest to our setting, privileged on-policy self-distillation conditions a teacher on verified solution information~\citep{zhao2026self}. SSOPD instead uses no external solution trace: the teacher is conditioned only on a correct completion from the same sampled group and is distilled into prefixes of an on-policy failure.

\section{Conclusion}

We introduced Self-Supervised On-Policy Distillation, an auxiliary objective that turns correct--wrong contrast inside a GRPO group into dense process supervision. SSOPD uses the shortest correct on-policy completion as a self-generated witness and distills a teacher conditioned on that witness into prefixes of the longest wrong completion, without requiring external solution traces. The stopping-time view motivates this finite-group rule, and the frontier weight concentrates the auxiliary loss on prompts where the sampled group contains meaningful branching. Empirically, SSOPD improves over GRPO on every evaluated model-benchmark pair and surpasses the solution-conditioned OPSD baseline in macro Avg@12 at 8B. These results suggest that on-policy groups contain reusable process information beyond their terminal rewards: a mixed group not only tells us which samples succeeded, but also exposes nearby failed prefixes where the model can be locally corrected. This perspective may be useful for other verifier-based training settings where solution traces are scarce but multiple attempts per prompt are already collected.

A mild limitation is that SSOPD relies on reliable answer verification and useful correct witnesses from the current policy, which may be weaker outside mathematical reasoning or for smaller models. Future work can improve witness selection, combine self-generated witnesses with occasional external traces, and study finer prefix selection rules. This paper presents work whose goal is to advance the field of machine learning. We do not foresee specific negative societal consequences.

\clearpage
\bibliographystyle{plainnat}
\bibliography{cite}

\clearpage
\onecolumn
\appendix

\section{Training Algorithm}
\label{sec:appendix_training_algorithm}

Algorithm~\ref{alg:ssopd} summarizes the SSOPD update using the same notation as \S\ref{sec:method}. A prompt $x$ produces a GRPO group $\mathcal G(x)$ under the behavior policy $\pi$. If the group contains both correct and wrong completions, the shortest correct completion $y^+$ is used as privileged hindsight, and the first $K_-$ prefixes of the longest wrong completion $y^-$ are edited by distillation.

\begin{algorithm*}[t]
\caption{SSOPD auxiliary update}
\label{alg:ssopd}
\begin{algorithmic}[1]
\Require Prompt minibatch $\mathcal B$, behavior policy $\pi$, student policy $\pi_\theta$, stop-gradient teacher $\pi_{\bar\theta}$, verifier $R$, group size $G$, prefix budget $K$, base coefficient $\lambda_0$, clipping threshold $\tau_{\mathrm{clip}}$
\ForAll{$x\in\mathcal B$}
    \State Sample $\mathcal G(x)=\{y_i\}_{i=1}^G$, with $y_i\sim\pi(\cdot\mid x)$
    \State Compute rewards $r_i\gets R(y_i)$ and the GRPO loss $\mathcal L_{\mathrm{GRPO}}(x)$
    \State $\mathcal C(x)\gets\{y_i:r_i=1\}$, $\mathcal W(x)\gets\{y_i:r_i=0\}$
    \If{$|\mathcal C(x)|>0$ and $|\mathcal W(x)|>0$}
        \State $y^+\gets\arg\min_{y\in\mathcal C(x)}\tau(y)$
        \State $y^-\gets\arg\max_{y\in\mathcal W(x)}\tau(y)$
        \State $K_-\gets\min(K,\tau(y^-))$, $\hat p_x\gets |\mathcal C(x)|/G$
        \State $\lambda_x\gets\lambda_0\cdot 4\hat p_x(1-\hat p_x)$ \Comment{stop-gradient}
        \For{$t=0,\ldots,K_- -1$}
            \State $s_t^-\gets (x,y^-_{<t})$
            \State $q_T(\cdot\mid s_t^-,y^+)\gets\pi_{\bar\theta}(\cdot\mid\operatorname{Hint}(s_t^-,y^+))$
            \State $p_\theta(\cdot\mid s_t^-)\gets\pi_\theta(\cdot\mid s_t^-)$
        \EndFor
        \State $\mathcal L_{\mathrm{SSOPD}}(x)\gets\frac1{K_-}\sum_{t=0}^{K_- -1}\ell_{\mathrm{OPSD}}(s_t^-,y^+)$
    \Else
        \State $\lambda_x\gets0$, $\mathcal L_{\mathrm{SSOPD}}(x)\gets0$
    \EndIf
\EndFor
\State Update $\theta$ using $\frac1{|\mathcal B|}\sum_{x\in\mathcal B}\left[\mathcal L_{\mathrm{GRPO}}(x)+\lambda_x\mathcal L_{\mathrm{SSOPD}}(x)\right]$
\end{algorithmic}
\end{algorithm*}

\section{Detailed Proofs}
\label{sec:detailed_proofs}

\paragraph{Local desirability edit.}
\begin{proof}[Proof of Theorem~\ref{thm:desirability_edit}]
Only the action distribution at state $s$ is edited, so
\[
V_\phi^{\pi_\eta}(s)=\sum_a \pi_\eta(a\mid s)Q_\phi^\pi(s,a).
\]
Substituting $\pi_\eta(a\mid s)=(1-\eta)\pi(a\mid s)+\eta q_\phi(a\mid s)$ gives
\[
V_\phi^{\pi_\eta}(s)
=(1-\eta)V_\phi^\pi(s)
+
\eta\sum_a q_\phi(a\mid s)Q_\phi^\pi(s,a).
\]
Using
\[
q_\phi(a\mid s)=\frac{\pi(a\mid s)Q_\phi^\pi(s,a)}{V_\phi^\pi(s)},
\]
we obtain
\[
\sum_a q_\phi(a\mid s)Q_\phi^\pi(s,a)
=
\frac{\mathbb E_{a\sim\pi(\cdot\mid s)}[Q_\phi^\pi(s,a)^2]}{V_\phi^\pi(s)}.
\]
Therefore
\[
V_\phi^{\pi_\eta}(s)-V_\phi^\pi(s)
=
\eta\left(
\frac{\mathbb E[Q_\phi^\pi(s,a)^2]}{V_\phi^\pi(s)}
-V_\phi^\pi(s)
\right).
\]
Since $V_\phi^\pi(s)=\mathbb E_{a\sim\pi(\cdot\mid s)}[Q_\phi^\pi(s,a)]$, this equals
\[
\eta
\frac{
\mathbb E[Q_\phi^\pi(s,a)^2]
-
\mathbb E[Q_\phi^\pi(s,a)]^2
}{V_\phi^\pi(s)}
=
\eta
\frac{\operatorname{Var}_{a\sim\pi(\cdot\mid s)}[Q_\phi^\pi(s,a)]}{V_\phi^\pi(s)}.
\]
The variance is nonnegative, proving the claim.
\end{proof}

\paragraph{Approximate teacher-induced edit.}
For a self-generated witness $z$, define the teacher approximation error
\[
\varepsilon_{\phi,s}(z)
=
\left|
\mathbb E_{a\sim q_T(\cdot\mid s,z)}[Q_\phi^\pi(s,a)]
-
\mathbb E_{a\sim q_\phi(\cdot\mid s)}[Q_\phi^\pi(s,a)]
\right|.
\]

\begin{theorem}[Approximate SSOPD edit]
\label{thm:approx_ssopd_edit}
For any prefix state $s$ with $V_\phi^\pi(s)>0$, let
\[
\tilde\pi_\eta(a\mid s)=(1-\eta)\pi(a\mid s)+\eta q_T(a\mid s,z).
\]
Then
\[
V_\phi^{\tilde\pi_\eta}(s)-V_\phi^\pi(s)
\ge
\eta
\left(
\frac{\operatorname{Var}_{a\sim\pi(\cdot\mid s)}[Q_\phi^\pi(s,a)]}{V_\phi^\pi(s)}
-
\varepsilon_{\phi,s}(z)
\right).
\]
\end{theorem}

\begin{proof}[Proof of Theorem~\ref{thm:approx_ssopd_edit}]
For the teacher-induced local edit,
\[
V_\phi^{\tilde\pi_\eta}(s)-V_\phi^\pi(s)
=
\eta\left(
\mathbb E_{a\sim q_T(\cdot\mid s,z)}[Q_\phi^\pi(s,a)]
-
V_\phi^\pi(s)
\right).
\]
Add and subtract the ideal posterior expectation:
\[
V_\phi^{\tilde\pi_\eta}(s)-V_\phi^\pi(s)
=
\eta\left(
\mathbb E_{a\sim q_\phi(\cdot\mid s)}[Q_\phi^\pi(s,a)]
-
V_\phi^\pi(s)
\right)
+
\eta\Delta_T,
\]
where
\[
\Delta_T=
\mathbb E_{a\sim q_T(\cdot\mid s,z)}[Q_\phi^\pi(s,a)]
-
\mathbb E_{a\sim q_\phi(\cdot\mid s)}[Q_\phi^\pi(s,a)].
\]
By Theorem~\ref{thm:desirability_edit}, the first parenthesized term is
\[
\frac{\operatorname{Var}_{a\sim\pi(\cdot\mid s)}[Q_\phi^\pi(s,a)]}{V_\phi^\pi(s)}.
\]
By definition of $\varepsilon_{\phi,s}(z)$, $\Delta_T\ge -\varepsilon_{\phi,s}(z)$. Combining the two displays proves the bound.
\end{proof}

\paragraph{Stopping-time objective and finite-group witnesses.}
We instantiate the generic desirability with a random deadline. Let $D$ be independent of the rollout and satisfy $\Pr(D\ge \ell)=\gamma^\ell$ for $0<\gamma<1$. For a continuation from prefix $s$, define
\[
\phi_F^s(Y)=\mathbf 1\{R(Y)=1\}\gamma^{L_s(Y)}.
\]
Let $V_F^\pi$ and $Q_F^\pi$ be the corresponding value and action-value, and let
\[
q_F(a\mid s)=\frac{\pi(a\mid s)Q_F^\pi(s,a)}{V_F^\pi(s)}
\]
be the fast-success posterior.

The same deadline induces the persistent-failure weight
\[
\psi_P(Y)=\mathbf 1\{R(Y)=0\}\left(1-\gamma^{\tau(Y)}\right).
\]
If $K_Y=\min(K,\tau(Y))$ and $\mathcal P_K(Y)=\{s_t(Y)\}_{t=0}^{K_Y-1}$, the corresponding ideal prefix distribution is
\[
\mu_P(s\mid x)
=
\frac{
\mathbb E_\pi\!\left[
\psi_P(Y)\frac1{K_Y}\sum_{t=0}^{K_Y-1}\mathbf 1\{s_t(Y)=s\}
\middle|x\right]
}{
\mathbb E_\pi[\psi_P(Y)\mid x]
}.
\]
The ideal stopping-time objective samples prefixes from $\mu_P$ and edits them toward $q_F$:
\[
\mathcal L_{\mathrm{ideal}}^{F/P}(x)
=
\mathbb E_{s\sim\mu_P(\cdot\mid x)}
\left[
D_{\mathrm{KL}}\!\left(q_F(\cdot\mid s)\middle\|p_\theta(\cdot\mid s)\right)
\right].
\]
Within a finite group, maximizing the root-level weights $\mathbf 1\{R(y)=1\}\gamma^{\tau(y)}$ and $\mathbf 1\{R(y)=0\}(1-\gamma^{\tau(y)})$ yields the shortest correct completion $y^+$ and the longest wrong completion $y^-$ used in SSOPD.

\begin{proof}[Proof of Theorem~\ref{thm:stopping_opsd_bound}]
For every selected prefix $s_t^-\in\mathcal P_K(y^-)$, apply Theorem~\ref{thm:approx_ssopd_edit} with $\phi=\phi_F$ and $z=y^+$. This gives
\[
V_F^{\tilde\pi_\eta^{s_t^-}}(s_t^-)-V_F^\pi(s_t^-)
\ge
\eta\left(
\frac{\operatorname{Var}_{a\sim\pi(\cdot\mid s_t^-)}[Q_F^\pi(s_t^-,a)]}{V_F^\pi(s_t^-)}
-
\varepsilon_{F,s_t^-}(y^+)
\right).
\]
Averaging this inequality over $t=0,\ldots,K_- -1$ yields the stated lower bound.
\end{proof}

\paragraph{Frontier weighting.}

\begin{proof}[Proof of Proposition~\ref{prop:frontier_coefficient}]
Let $S_t$ be the random prefix state at depth $t$, padded after termination by an absorbing terminal state, and define
\[
M_t=V_R^\pi(S_t)=\Pr_\pi(R=1\mid S_t).
\]
Then $(M_t)_{t=0}^H$ is a martingale. Indeed, conditioning on $S_t=s$,
\[
\mathbb E[M_{t+1}\mid S_t=s]
=\sum_a \pi(a\mid s)V_R^\pi(S_{t+1})
=\sum_a \pi(a\mid s)Q_R^\pi(s,a)
=V_R^\pi(s)=M_t.
\]
At the root, $M_0=p_x$; at the terminal padded state, $M_H=R$. Orthogonality of martingale increments gives
\[
\operatorname{Var}(R\mid x)
=\mathbb E[(M_H-M_0)^2\mid x]
=\sum_{t=0}^{H-1}\mathbb E[(M_{t+1}-M_t)^2\mid x].
\]
Conditioning on $S_t=s$, the next value is $Q_R^\pi(s,a)$ after action $a$, so
\[
\mathbb E[(M_{t+1}-M_t)^2\mid S_t=s]
=
\operatorname{Var}_{a\sim\pi(\cdot\mid s)}[Q_R^\pi(s,a)].
\]
Taking expectation over $S_t$ and using $\operatorname{Var}(R\mid x)=p_x(1-p_x)$ for binary $R$ proves the branching-variance identity.

For the finite-group contrast count, if $K_c\sim\operatorname{Binomial}(G,p_x)$, then
\[
\mathbb E[K_c(G-K_c)]
=G\mathbb E[K_c]-\mathbb E[K_c^2].
\]
Using $\mathbb E[K_c]=Gp_x$ and $\mathbb E[K_c^2]=G p_x(1-p_x)+G^2p_x^2$, we get
\[
\mathbb E[K_c(G-K_c)]
=G(G-1)p_x(1-p_x).
\]
\end{proof}

\section{Implementation Details}
\label{sec:experimental_details}

The auxiliary term is computed from the same sampled group used by GRPO; no additional rollout policy is introduced. In the notation of \S\ref{sec:method}, $\pi$ is the old policy that generates the group, $\pi_\theta$ is the LoRA-updated student, and $\pi_{\bar\theta}$ is the fixed teacher used only for target construction. Teacher forward passes are run with stop-gradient. The student receives gradients from both $\mathcal L_{\mathrm{GRPO}}$ and $\mathcal L_{\mathrm{SSOPD}}$. For a selected wrong completion $y^-$, we evaluate the auxiliary loss on prefix states $s_t^-=(x,y^-_{<t})$ for $t<K_-$.

The prompt-level weight uses the raw group success rate
\[
\hat p_x=\frac{|\mathcal C(x)|}{G},
\qquad
\lambda_x=\lambda_0\cdot 4\hat p_x(1-\hat p_x),
\]
and $\lambda_x$ is treated as stop-gradient. No extra correction is applied to this estimate. If the group is all-correct or all-wrong, then $\hat p_x(1-\hat p_x)=0$, so the auxiliary term is automatically inactive.

Unless otherwise stated, the GRPO+SSOPD runs train LoRA adapters with learning rate $5\times10^{-6}$ and effective batch size $32$ and experiments are run on A100 or H100 GPUs. We sample $G=8$ completions per prompt, which keeps each prompt group on one process and makes the correct--wrong pair selection local. The LoRA rank is $r=64$ with $\alpha=128$; the adapted modules are q\_proj, k\_proj, v\_proj, o\_proj, gate\_proj, up\_proj, and down\_proj.

The GRPO rollout configuration uses binary correctness rewards, group reward scaling, KL coefficient $\beta=0$, sampling temperature $1.2$, maximum prompt length $2048$, and maximum rollout length $16000$. The SSOPD regularizer uses $\lambda_0=0.5$, the fixed-teacher setting obtained by disabling trainable adapters during teacher scoring, and the shortest-correct / longest-wrong selection rule. We apply the regularizer to at most the first $K=1024$ generated positions of $y^-$, compute full-vocabulary teacher targets with vocabulary chunk size $8192$, and clip each pointwise distillation summand at $0.05$. Checkpoints are saved every $50$ steps; the HMMT 2025 sweep reports checkpoints up to step $250$.

Evaluation follows the Qwen3 recommended sampling setup with thinking mode enabled: temperature $1.0$, top-p $0.95$, top-k $-1$, min-p $0.0$, no presence penalty, maximum generation length $38{,}912$, and $12$ samples per prompt.

\section{Dataset Licenses and Terms}
\label{sec:dataset_licenses}

Table~\ref{tab:dataset-licenses} summarizes the data assets used for training and evaluation. We checked the public dataset cards and source pages before submission: \url{https://huggingface.co/datasets/siyanzhao/Openthoughts_math_30k_opsd}, \url{https://huggingface.co/datasets/HuggingFaceH4/aime_2024}, \url{https://huggingface.co/datasets/yentinglin/aime_2025}, and \url{https://huggingface.co/datasets/MathArena/hmmt_feb_2025}. The training set is a filtered OpenThoughts math subset released as \texttt{siyanzhao/Openthoughts\_math\_30k\_opsd}; its Hugging Face card does not list a separate license, while the upstream OpenThoughts project and OpenThoughts datasets are released under Apache-2.0. The AIME datasets are used only for evaluation. Their Hugging Face cards link to Art of Problem Solving pages for the problem statements but do not specify an explicit dataset license; we therefore treat them as evaluation-only assets with original contest rights retained by their respective owners. The HMMT February 2025 package from MathArena explicitly lists CC BY-NC-SA 4.0.

\begin{table}[h]
\centering
\caption{Datasets and benchmark assets used in the experiments.}
\label{tab:dataset-licenses}
\small
\begin{tabular}{@{}p{0.21\textwidth}p{0.29\textwidth}p{0.16\textwidth}p{0.25\textwidth}@{}}
\toprule
\textbf{Asset} & \textbf{Source} & \textbf{Use} & \textbf{License / terms} \\
\midrule
OpenThoughts math subset & \texttt{siyanzhao/}\newline\texttt{Openthoughts\_math\_}\newline\texttt{30k\_opsd}; upstream OpenThoughts: \url{https://github.com/open-thoughts/open-thoughts} & Training prompts and reference traces for privileged baselines & Derived dataset card has no separate license field; upstream OpenThoughts artifacts are Apache-2.0. \\
AIME 2024 & \texttt{HuggingFaceH4/aime\_2024}, sourced from AIME I/II 2024 problem pages on Art of Problem Solving & Evaluation only & No explicit license field on the dataset card; original contest/problem rights retained by the original owners. \\
AIME 2025 & \texttt{yentinglin/aime\_2025}, sourced from AIME I/II 2025 problem pages on Art of Problem Solving & Evaluation only & No explicit license field on the dataset card; original contest/problem rights retained by the original owners. \\
HMMT February 2025 & \texttt{MathArena/hmmt\_feb\_2025}; MathArena benchmark collection: \url{https://matharena.ai/} & Evaluation only & CC BY-NC-SA 4.0 as stated on the MathArena Hugging Face dataset card. \\
\bottomrule
\end{tabular}
\end{table}

\section{Additional Ablations}
\label{sec:additional_ablations}

\paragraph{Selector ablation.}
We ablate the choice of selectors by enumerating the Cartesian product between four correct-side selectors and three wrong-side selectors on HMMT 2025. We denote selecting the longest and shortest sequence as $\mathrm{Len}_{\max}$ and $\mathrm{Len}_{\min}$, respectively. $\mathrm{AvgLogP}_{\max}$ and $\mathrm{AvgLogP}_{\min}$ denote selecting the sequence with the maximum and minimum sequence-average log-probability. For each selector pair, we report the best \texttt{Avg@12} across checkpoints.

\begin{table}[t]
\centering
\caption{Ablation of selector choices on HMMT 2025. Each entry reports the best \texttt{Avg@12} across checkpoints.}
\label{tab:selector-ablation}
\vspace{1mm}
\small
\begin{tabular}{@{}lccc@{}}
\toprule
& \multicolumn{3}{c}{\textbf{Wrong-side selector}} \\
\cmidrule(lr){2-4}
\textbf{Correct-side selector}
& $\mathrm{Len}_{\max}$
& $\mathrm{Len}_{\min}$
& $\mathrm{AvgLogP}_{\max}$ \\
\midrule
$\mathrm{Len}_{\max}$
& 42.8 & 44.4 & 44.4 \\
$\mathrm{Len}_{\min}$
& \textbf{45.0} & 44.2 & 44.2 \\
$\mathrm{AvgLogP}_{\max}$
& 42.5 & 42.8 & 42.8 \\
$\mathrm{AvgLogP}_{\min}$
& 42.8 & 42.8 & 42.8 \\
\bottomrule
\end{tabular}
\end{table}

The best selector pair is the one used by SSOPD: the shortest correct completion, $\mathrm{Len}_{\min}$ on the correct side, and the longest wrong completion, $\mathrm{Len}_{\max}$ on the wrong side. This supports the stopping-time view in \S\ref{sec:method}: short correct completions are stronger empirical witnesses of fast success, while long wrong completions identify persistent failure prefixes that benefit from local distillation.

\end{document}